%% file: main.tex
\documentclass{article}

\usepackage{microtype}
\usepackage{graphicx}
\usepackage{subfigure}
\usepackage{booktabs} 

\usepackage{hyperref}

\usepackage{amsmath}
\usepackage{amssymb}
\usepackage{amsthm}
\usepackage{amsfonts}

\newcommand{\avg}{\tfrac{1}{N}\sum_i}

\newtheorem{theorem}[]{Theorem}
\newtheorem{obs}[]{Observation}

\newtheorem{lemma}[]{Lemma}

\usepackage[accepted]{icml2019}

\definecolor{CommentCyan}{rgb}{0.0,0.7,0.7}

\icmltitlerunning{A Deep Neural Network's Loss Surface Contains Every Low-dimensional Pattern}

\begin{document}

\twocolumn[
\icmltitle{A Deep Neural Network's Loss Surface\\Contains Every Low-dimensional Pattern}




\begin{icmlauthorlist}
\icmlauthor{Wojciech Marian Czarnecki}{dm}
\icmlauthor{Simon Osindero}{dm}
\icmlauthor{Razvan Pascanu}{dm}
\icmlauthor{Max Jaderberg}{dm}
\end{icmlauthorlist}

\icmlaffiliation{dm}{DeepMind}
\icmlcorrespondingauthor{Wojciech Marian Czarnecki}{lejlot@google.com}

\icmlkeywords{Optimisation, Loss surface, Geometry, Universal approximation}

\vskip 0.3in
]



\printAffiliationsAndNotice{} 
\begin{abstract}
The work ``Loss Landscape Sightseeing with Multi-Point Optimization''~\citep{skorokhodov2019loss} demonstrated that one can empirically find arbitrary 2D binary patterns inside loss surfaces of popular neural networks. In this paper we prove that: (i) this is a general property of deep universal approximators; and (ii) this property holds for arbitrary smooth patterns, for other dimensionalities, for every dataset, and any neural network that is sufficiently deep and wide.
Our analysis predicts not only the existence of all such low-dimensional patterns, but also two other properties that were observed empirically: (i) that it is easy to find these patterns; and (ii) that they transfer to other data-sets (e.g. a test-set).
\end{abstract}


\input{content}

\bibliographystyle{abbrvnat}
\setlength{\bibsep}{5pt} 
\setlength{\bibhang}{0pt}
\bibliography{references}

\end{document}

%% file: content.tex
\begin{figure*}[h]
    \centering
    \includegraphics[width=0.8\textwidth]{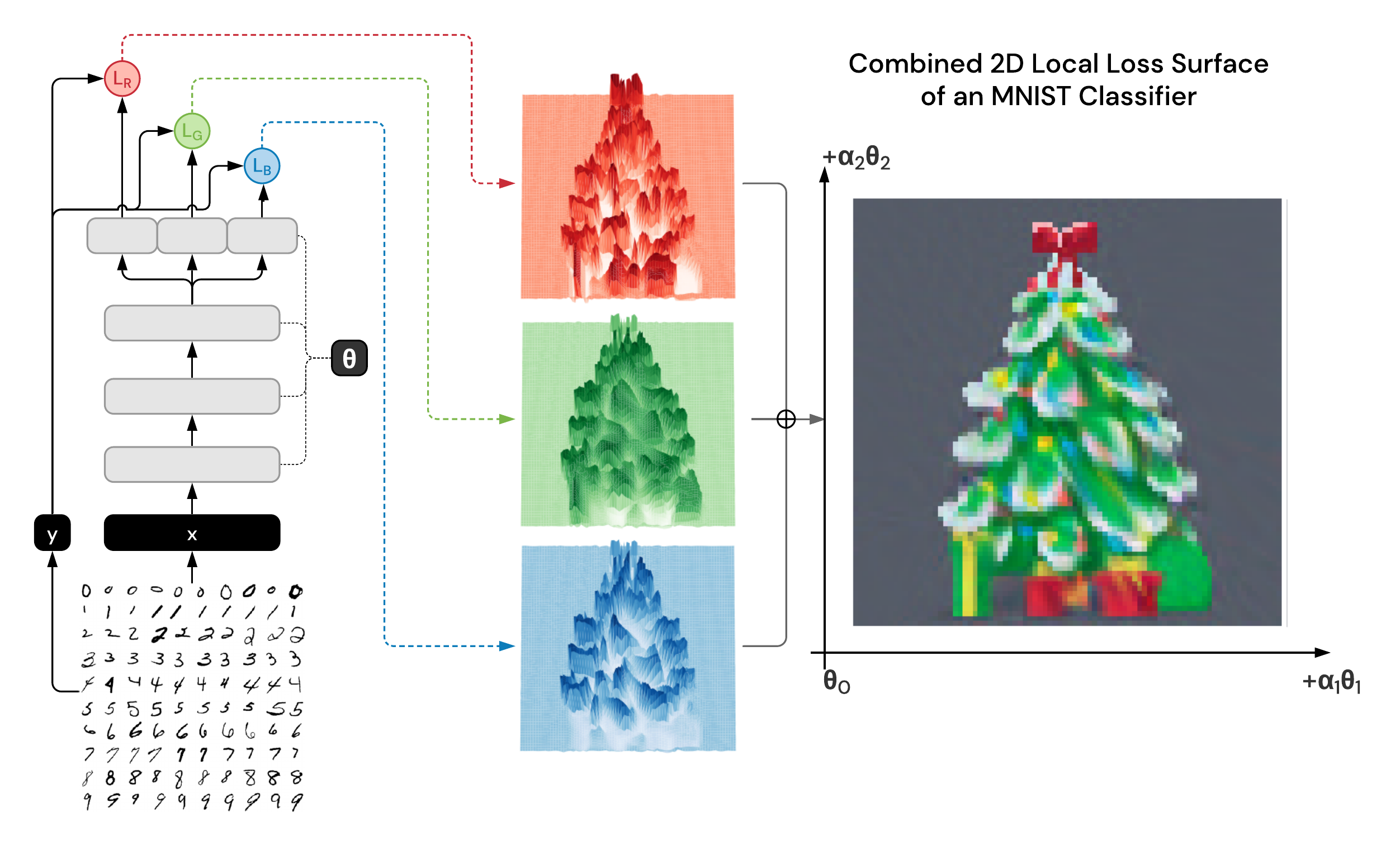}
    \caption{A local loss surface found in an MNIST classification task
    using a three hidden layer MLP, with three heads, each minimising the same classification loss, and visualised in three separate colour channels. For visual clarity, we squash the values through sigmoid before combining them in the final plot. $\theta_0$, $\theta_1$ and $\theta_2$ are three sets of parameters, found from the optimisation procedure using Theorem 1 ($\theta_0$ is the weight vector constructed with Eq.~\ref{eq:t0}, and $\theta_1$ and $\theta_2$ according to Eq.~\ref{eq:ti}). $\alpha_1$ and $\alpha_2$ are scalars, varied between 0 and 1 on the plot.}
    \label{fig:mnist_is_a_tree}
\end{figure*}
\section{Introduction}

Understanding the underlying geometry of the loss surfaces of neural networks is an active field of research~\citep{choromanska2015open,li2018visualizing,fort2019large}. Recently, \citet{skorokhodov2019loss} noticed an interesting phenomenon -- they were able to find arbitrary 2D binary shapes in loss landscapes using their proposed optimisation scheme. 
They also observed that this geometry transferred from training to test-set, and that it was relatively easy to find with first order optimisation.
These exciting empirical results open up theoretical questions about the nature of these observations, and such questions form the focus of this paper.
%
We provide a simple and constructive theoretical argument showing that every low-dimensional pattern can be found in the loss surface of any large enough network, and with every dataset -- see Figure~\ref{fig:mnist_is_a_tree} for an example. 
Furthermore, our theoretical results predict transfer to the test-set, as well as optimisation complexity that is no harder than typical supervised training of deep neural networks.
Finally, we show that these patterns can also be found around approximate global minima.

\section{Losses as universal approximators}


We will show that we can always find a $z$-dimensional section through the loss surface of a big enough neural network where the empirical loss of our network resembles a target function $\mathcal{T}$ 
up to some arbitrary constant bias.
In other words, we seek to find a set of parameters $\boldsymbol{\theta}_0$ and $z$ directions in the parameters space $\boldsymbol{\theta}_i$, such that the training loss of a network in each point of the hypercube $\{\boldsymbol{\theta_0} + \sum_i \alpha_i \boldsymbol{\theta}_i : \alpha_i \in [0, 1]\}$ is approximately equal to a predefined $\mathcal{T}$.
This high level idea is expressed in Figure~\ref{fig:overall}.

We begin with a simple lemma that allows us to add a rich class of activation functions at the top of neural network, without affecting its power as a function-approximator.
\begin{lemma}
If a parametrised family of functions $\mathcal{F} := \{f_{\boldsymbol{\theta}} : [0,1]^d \rightarrow [0,1] : {\boldsymbol{\theta}} \in {\boldsymbol{\Theta}}\}$ defines universal approximators of continuous functions over the same domain, then a transformed family $\sigma(\mathcal{F}) := \{\sigma(f) : f \in \mathcal{F}\}$   defines universal approximators over continuous functions from $[0,1]^d \rightarrow [a,b]$ if $\sigma: [0,1] \rightarrow [a,b]$ is Lipschitz and surjective.
\end{lemma}
\begin{proof}
Comes directly from the Lipschitz property, which bounds changes in the output by changes in the input, and from the fact that surjectivity means one can represent every value in $[a,b]$.
\end{proof}

Equipped with this lemma, we can prove the main result of this paper.
\begin{figure*}[htb]
    \centering
    \includegraphics[width=\textwidth]{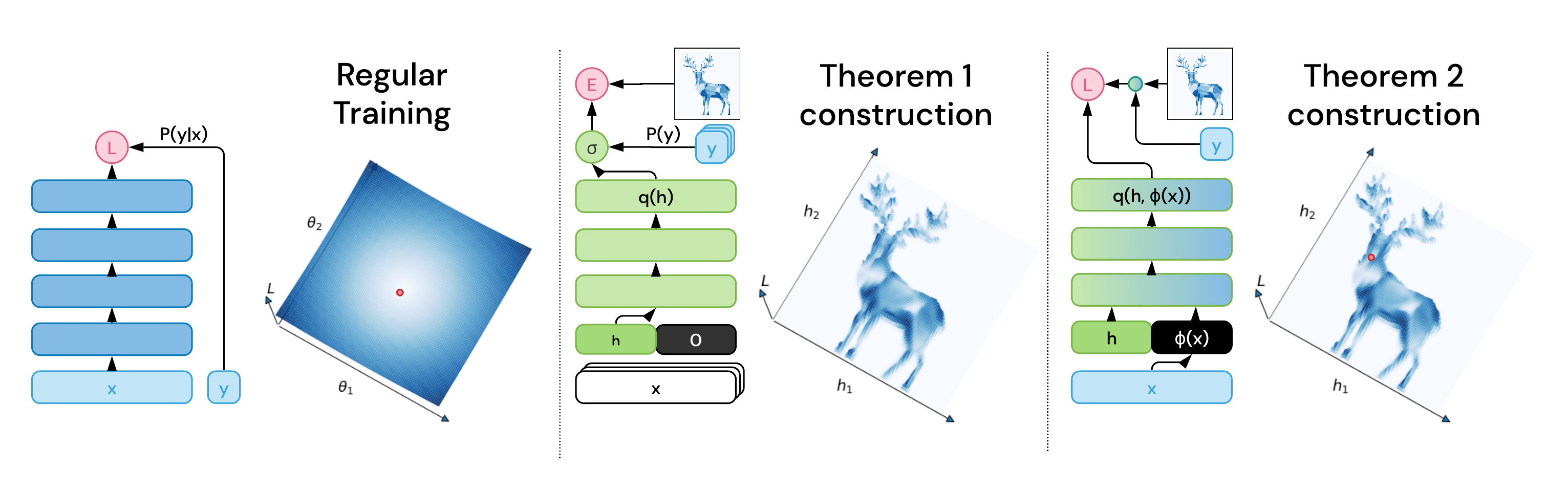}
    \caption{Visualisation of the construction from each theorem. Blue blocks correspond to parts of the network trained to solve the underlying task (mapping from $x$ to $y$). Green blocks correspond to parts of the network used to approximate the target loss surface. Pink blocks are the losses one minimises. Black blocks represent fixed values, no longer trained. A white block represents an entity that is no longer affecting the model. A red dot is a global minimum.}
    \label{fig:overall}
\end{figure*}


\begin{theorem}
Every low-dimensional pattern can be found in a loss surface of sufficiently large\footnote{Having enough layers and/or hidden units in each.} deep neural network.
\end{theorem}
\begin{proof}
Let us begin by establishing our notation. We assume that the first layer of our network is parameterised by ${\boldsymbol{\theta}}_\mathrm{W}$ (weights) and ${\boldsymbol{\theta}}_\mathrm{b}$ (biases). The parameters of the remaining layers will be jointly referred to as ${\boldsymbol{\theta}}'$. The combination of all the parameters is ${\boldsymbol{\theta}}=\{{\boldsymbol{\theta}}_\mathrm{W}, {\boldsymbol{\theta}}_\mathrm{b}, {\boldsymbol{\theta}}'\}$. We further assume that ${\boldsymbol{\theta}}_\mathrm{W} \in \mathbb{R}^{d \times k}$ where $d$ is the dimensionality of the input, and that the number of hidden units in first layer, $k$, is such that $k \geq z$.
\noindent The output of our network is 
\begin{equation}
f_{\boldsymbol{\theta}}(\mathbf{x}) =f_{\{{\boldsymbol{\theta}}_\mathrm{W}, {\boldsymbol{\theta}}_\mathrm{b}, {\boldsymbol{\theta}}'\}}(\mathbf{x}) = g_{{\boldsymbol{\theta}}'}(\langle \mathbf{x}, {\boldsymbol{\theta}}_\mathrm{W} \rangle + {\boldsymbol{\theta}}_\mathrm{b} ),
\end{equation}
and the empirical loss summed over $N$ input vectors $\mathbf{x}_i$ and their corresponding targets $\mathbf{y}_i$ is
\begin{equation}
L({\boldsymbol{\theta}}) = \avg \ell(f_{\boldsymbol{\theta}}(\mathbf{x}_i), \mathbf{y}_i),
\end{equation}
for some $\ell: \mathbb{R}^l \times \mathbb{R}^l \rightarrow \mathbb{R}$ (e.g. cross entropy or Euclidean distance).
Without loss of generality we will focus on the univariate case of $l=1$, but exactly the same reasoning can be applied for any $l \in \mathbb{N}$ (see Figure~\ref{fig:mnist_is_a_tree} for an example with $l=3$).
Our target (bounded, smooth) z-dimensional loss pattern is
\begin{equation}
\mathcal{T}(h_1, ..., h_z) : [0,1]^z \rightarrow [0,1].
\end{equation}
%
%
Let us set ${\boldsymbol{\theta}}_\mathrm{W} = \mathbf{0}$ and ${\boldsymbol{\theta}}_\mathrm{b} = [h_1, ..., h_z, 0, ..., 0]$.
We can then define a new neural network:
\begin{equation}
\begin{aligned}
q_{{\boldsymbol{\theta}}'}(h_1, ..., h_z) := &f_{\{\mathbf{0},[h_1, ..., h_z, 0, ..., 0],{\boldsymbol{\theta}}'\}}(\mathbf{x})\\
= &g_{{\boldsymbol{\theta}}'}(\langle \mathbf{x}, \mathbf{0}\rangle + [h_1, ..., h_z, 0, ..., 0])\\
= &g_{{\boldsymbol{\theta}}'}([h_1, ..., h_z, 0, ..., 0]),
\end{aligned}
\end{equation}
and with such parameterisation, the output of $f$ is independent of its inputs.
Note that according to the universal approximation theorem~\citep{hornik1991approximation}, $q$ is still a universal approximator of $\mathcal{C}^1$ functions from $[0,1]^z$ to $[0,1]$. 
The only thing left to show is that our loss is also a universal approximator. 
However, recall that we have:
\begin{equation}
    \begin{aligned}
&L(\{\mathbf{0},[h_1, ..., h_z, 0, ..., 0],{\boldsymbol{\theta}}'\})\\ &= \avg \ell(f_{\{\mathbf{0},[h_1, ..., h_z, 0, ..., 0],{\boldsymbol{\theta}}'\}}(\mathbf{x}_i), \mathbf{y}_i)\\ &= \avg \ell(q_{{\boldsymbol{\theta}}'}(h_1,..., h_z), \mathbf{y}_i),    
    \end{aligned}
\end{equation}

\noindent and $q$ is a universal approximator.
Therefore we see, from Lemma 1, that for any $\mathcal{T}$ we can find parameters
${\boldsymbol{\theta}}_\epsilon(\mathcal{T})$ to replace 
%
${\boldsymbol{\theta}}'$,
such that the loss surface given by $L$
becomes arbirarily close to $\mathcal{T}$.
In order for Lemma 1 to be valid, we simply require that the function $\sigma(p) := \avg \ell(p, \mathbf{y}_i)$, representing the loss on top of the neural network,  is (up to a constant) surjective in the set of target values, meaning that using notation of $\mathrm{Im}(f) = \{y: \exists \mathbf{x}\; f(\mathbf{x})=y\}$
$$
\exists_{a \in \mathbb{R} }\;\;\mathrm{Im}(\sigma) + a \supseteq \mathrm{Im}(\mathcal{T}),
$$
and that it is locally Lipschitz (so that small changes in $p$ correspond to small changes in $L$). These properties are satisfied by any reasonable combinations of neural network loss function and target pattern $\mathcal{T}$. For example let us choose a quadratic form of loss $\ell(p,y) = \ell_2(p,y) = \|p-y\|^2$, then\footnote{To simplify notation, we assume the output of the network is 1-dimensional. The multi-dimensional case is identical, with the loss summing over the output dimensions.}
\begin{equation}
\begin{aligned}
\sigma_\mathrm{{L_2}}(p) & = \avg \ell_2(p,y_i) = \avg (p - y_i)^2 \\
& = p^2 - \left (\avg 2y_i\right ) p + \avg y_i^2,
\end{aligned}
\end{equation}
is still a quadratic with parameters that are the only quantities in the entire construction that now depend on the training dataset. Thus, up to a constant offset, we have
$
\mathrm{Im}(\sigma_\mathrm{{L_2}}) \supseteq [0,1] = \mathrm{Im}(\mathcal{T})
$.
Furthermore, quadratic functions are locally Lipschitz.
Thus, one can think about this function as a ``final activation'' of our regressor; just one that happens to be parameterised with some arbitrary constants. 
Moving beyond quadratic losses, we can show analogous properties for cross entropy and other standard losses (e.g. Figure~\ref{fig:implicit_act}).
\end{proof}

The above theorem can easily be adapted to convolutional neural networks. The construction is identical, with the requirement that the number of feature maps on the first layer is larger than $z$. We can remove the dependence on the input in the same way, and use the bias of the first convolutional layer to derive the $z$-dimensional subspace. As relying on convolutional layers does not affect the universal approximator properties of the model, everything else stays analogous. 

As a result of Theorem 1 we know that, with a large enough network, the target loss surface can be found at 
\begin{equation}
{\boldsymbol{\theta}}_\epsilon^*(\mathbf{h}):= \{\mathbf{0},[h_1, ..., h_z, 0, ..., 0],{\boldsymbol{\theta}}_\epsilon(\mathcal{T})\},
\end{equation}
for $\mathbf{h} \in [0,1]^z$. 
This means that for each $\epsilon > 0$ we have 
\begin{equation}
 \min_{a\in \mathbb{R}} \int_{[0,1]^z} \| \mathcal{T}(\mathbf{h}) - L({\boldsymbol{\theta}}_\epsilon^*(\mathbf{h})) + a\|^2
\mathrm{d}\mathbf{h} < \epsilon.
\end{equation}
Equivalently, we can say that we found the initial position
 \begin{equation}
 \label{eq:t0}
 \boldsymbol{\theta}_0 =  \{\mathbf{0},[0, ..., 0, 0, ..., 0],{\boldsymbol{\theta}}_\epsilon(\mathcal{T})\},\end{equation}
and directions 
 \begin{equation}
 \label{eq:ti}
 \boldsymbol{\theta}_i = \{\mathbf{0},[0, ..., \underset{i}{1}, 0, ..., 0],\mathbf{0}\},
 \end{equation}
so that
\begin{equation}
L(\boldsymbol{\theta}_0 + \sum_i \alpha_i \boldsymbol{\theta}_i) \approx \mathcal{T}(\alpha_1, ..., \alpha_z),
\end{equation}
for $\alpha_i \in [0,1]$ (see Figure~\ref{fig:mnist_is_a_tree} for an illustration). This transformation provides the equation for parameters of the original network $f_{\boldsymbol{\theta}}$ where the pattern can be found.

Interestingly the loss surface constructed in this way is axis aligned, but an analogous construction could be done for any other hidden layer, for other units of the same layer, for arbitrary combinations of units, and so on.

\begin{figure}[htb]
    \centering
    \includegraphics[width=0.24\textwidth]{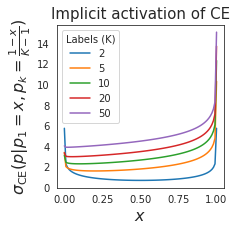}
    \includegraphics[width=0.224\textwidth]{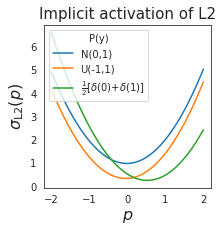}
    \caption{Visualisations of implicit activations functions $\sigma$ for cross entropy loss (left) and squared loss (right).}
    \label{fig:implicit_act}
\end{figure}

Summarising thus far -- we have shown that the loss surface of a neural network is a universal approximator of low-dimensional functions wrt. some of its own parameters, provided that: (i) the network has one linear layer acting as an information bottleneck
(e.g. any MLP); and that (ii) the sub-network downstream of this layer is itself still a universal approximator\footnote{In fact, this need not be a sub-network -- any other universal approximator would suffice, e.g. a Gaussian Process.}; and finally that (iii) the loss function is non-degenerate (which allows us to represent all values of interest, after marginalising over training set labels).
This implies that whatever low-dimensional geometry one chooses, we will be able to find a corresponding low-dimensional subspace of parameter space that will resemble this geometry to arbitrarily low error (where the size of the error is a function of size of the model).

\begin{obs} 
Patterns obtained using construction from Theorem 1 transfer to any samples with the same $P(\mathbf{y})$ (e.g. typical test sets of datasets).
\end{obs}
\begin{proof}
This comes directly from the fact that the prediction, $L({\boldsymbol{\theta}}_\epsilon^*(\mathbf{h}))$, depends on the dataset only through the marginal distribution of the output labels,  $P(\mathbf{y})$, and this in turn affects the definition of the implicit activation function, $\sigma$. Since this marginal distribution, $P(\mathbf{y})$, is (approximately) the same for both the training and test sets, the values for the projected loss surface will also be (approximately) the same. Of course by changing the marginals, $P(\mathbf{y})$, for the test-set one can break this property.
\end{proof}

A consequence of the above observation is that the loss surface patterns will also transfer to many other datasets, as long as the loss function used is the same, and the marginal $P(\mathbf{y})$ matches. For instance, every single perfectly balanced classification dataset with some predefined number of classes (e.g.: MNIST, Fashion MNIST, and CIFAR-10 in the case of 10-classes) will exhibit the same set of patterns (constructed via Theorem 1) for a given network..

\begin{obs}
Finding ${\boldsymbol{\theta}}^*$ is no harder than standard loss optimisation for typical neural networks.
\end{obs}
\begin{proof}
Using the construction from Theorem 1 leads to recasting the problem as supervised learning with a slightly non-standard final activation function. However, since this activation function is not used within the hidden layers of the model, it will not affect the learning process beyond, effectively, relabelling the output targets. The optimisation problem ends up being equivalent to
\begin{equation}
\underset{\theta \in \Theta}{\arg\min}\;\mathbb{E}_{\mathbf{h} \sim \mathrm{U}([0,1]^z)} \| q_\theta(\mathbf{h}) - \sigma^{-1}(\mathcal{T}(\mathbf{h})) \|^2,
\end{equation}
where $\sigma^{-1}$ can be defined as a maximum element in the pre-image of $\sigma$ to make the mapping unique (since $\sigma$ is often only locally injective). This is a standard objective of supervised learning, and thus the problem is equally hard.
\end{proof}

A natural question to ask is whether we can break the loss-surface-fitting solution's independence from $\mathbf{x}$ in a way that allows us to also still have solutions of the original task $L$ somewhere in our section. Quite surprisingly, we can now extend our construction to this case as well.

\begin{figure*}[htb]
    \centering
    \includegraphics[width=\textwidth]{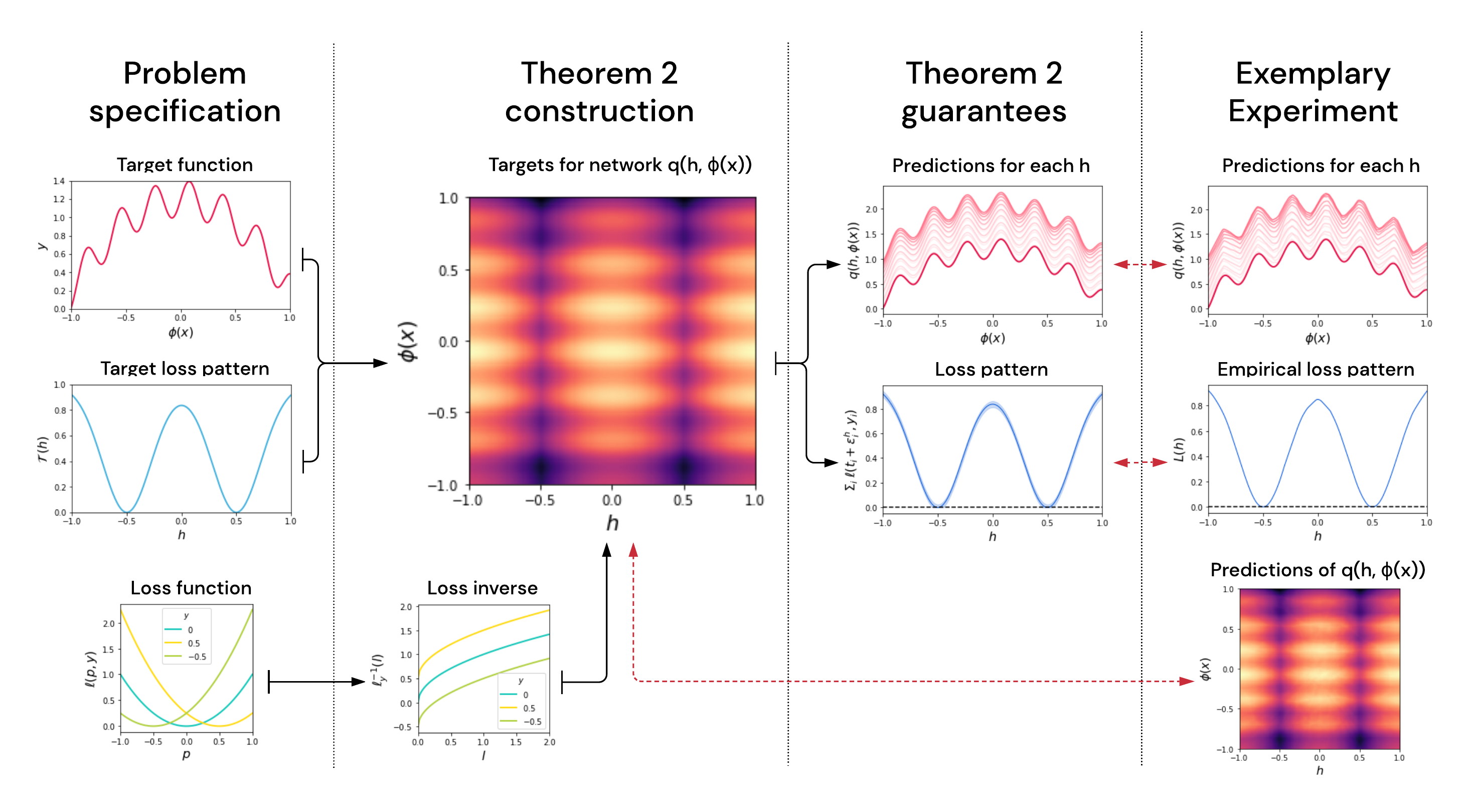}
    \caption{Visualisation of the construction from Theorem 2, on a simple example of 1D regression from $\phi(x)$ to $-\phi(x)^2 + \tfrac{\sin(20\phi(x))}{5} + 1.2$ and 1D target loss pattern $\mathcal{T}(h) = 1-(\exp(-\tfrac{(h-0.5)^2}{0.1}) + \exp(-\tfrac{(h+0.5)^2}{0.1}))$. For simplicity of illustration, we present the target function for this toy example as being based on $\phi(x)$, rather than $x$ itself -- due to the assumed injectivity, this incurs no loss of generality. We use the quadratic loss $\ell(p,y) = (p-y)^2$. On the right hand side one sees that effectively, our construction forces the network to build a distribution over predictions, each being slightly shifted, so that after transforming through the loss calculation -- they correspond to changes in the target loss pattern. There are two minima, $h^*=\pm 0.5$, that are realised in the resulting model. We provide an empirical result too, with an MLP trained with the construction from Theorem 2, that shows both replication of predictions, as well as a pattern, with the correct placement of minima.}
    \label{fig:thm2}
\end{figure*}
\begin{theorem}
Every low-dimensional pattern can be found in a loss surface of a sufficiently deep neural network, and within the pattern there exists a point that leads to a loss that is within epsilon\footnote{The value of epsilon depends on the size of the network, and can be made arbitrarily small as the network size grows.} of the global minimum.
\end{theorem}
\begin{proof}

The construction of this argument is very similar to Theorem 1. We take ${\boldsymbol{\theta}}_\mathrm{b} = [h_1, ..., h_z, 0, ..., 0]$ and for ${\boldsymbol{\theta}}_\mathrm{W} \in \mathbb{R}^{d \times k}$ we set to zero the columns 1 through $z$ (so that $h_1, ... h_z$ are independent from $\mathbf{x}$), while the remaining columns are left to be randomly initialised. 

We also assume that this initial shattering is injective, so that we can forget about $\mathbf{x}$ and instead work with its $(k-z)$-dimensional embedding, $\phi(\mathbf{x})$, in the first hidden layer without losing our ability to represent a minimum of $L$. 
Then, in a similar fashion to the previous construction, the desired loss-surface section will be axis aligned and affecting $\mathbf{h}$ through control of biases of the first layer. 
Let us again consider this setup as a new neural network:
\begin{equation}
\begin{aligned}
q_{{\boldsymbol{\theta}}'}(\mathbf{h}, \phi(\mathbf{x})) :=& f_{\{{\boldsymbol{\theta}}_\mathrm{W},  \boldsymbol{\theta}_b, \boldsymbol{\theta}'\}}(\mathbf{x})\\
= &f_{\{{\boldsymbol{\theta}}_\mathrm{W}, [h_1, ..., h_z, 0, ..., 0], {\boldsymbol{\theta}}'\}}(\mathbf{x}).
\end{aligned}
\end{equation}%
We note that it is a universal approximator in its input space. Consequently, for any mixture of $h_i$ and an input sample, it can produce an output arbitrarily close to any target we specify.

The (approximate) global minimum of $L$ will be located in the point representing the minimum of the target loss pattern $\mathcal{T}$: let us call this point $\mathbf{h}^* := \arg\min_{\mathbf{h}\in[0,1]^z} \mathcal{T}(\mathbf{h})$.
Since all our loss-pattern-matching arguments hold up to a constant, we can assume that value of the loss in this minimum is equal to the minimum in the pattern $ L^* = \mathcal{T}(\mathbf{h}^*)$.

For simplicity of the argument, we assume that a loss $\ell_i(p) = \ell(p, y_i)$ is invertible and the inverse itself is smooth (which is a common property, which one can check using Inverse Function Theorem), and for losses like Euclidean distance, we simply restrict the inverse operation to a positive part only.

Under this assumption, let us define the targets for network $q$ to be
\begin{equation}
q_{\boldsymbol{\theta}'}(\mathbf{h}, \phi(\mathbf{x}_i)) \approx  \ell_i^{-1}( \mathcal{T}(\mathbf{h})),
\end{equation}
see Figure~\ref{fig:thm2} for visualisation of this construction.
Thanks to the assumptions of $\ell_i^{-1}$ smoothness, this target is also smooth as a composition of smooth functions,\footnote{Smoothness is required only for the general case of approximating in functional space, for a finite grid it is trivially satisfied as there always exists a smooth function going through these.} thus, as before, we can expect to be able to approximate it to arbitrary precision with a large enough network.

Let us now compute the total loss $L(\mathbf{h})$ of this construction, using $\epsilon^\mathbf{h}_i \in \mathbb{R}$ to denote the approximation error at sample $\mathbf{x}_i$ and $\mathbf{h}$:
\begin{equation}
\begin{aligned}
L(\mathbf{h}) &= \avg \ell( q_{\boldsymbol{\theta}'}(\mathbf{h}, \phi(\mathbf{x}_i)), \mathbf{y}_i) \\
&= \avg \ell_i( 
\underset{\mathbf{t}_i}{\underbrace{ \ell_i^{-1}(\mathcal{T}(\mathbf{h}))}} + \epsilon_i^\mathbf{h}
),
\end{aligned}
\end{equation}
and due to construction: 
\begin{equation}
\begin{aligned}
\avg \ell_i(\mathbf{t}_i) &= \avg \ell_i( \ell_i^{-1}(\mathcal{T}(\mathbf{h}))) \\
&= \avg (\ell_i \circ  \ell_i^{-1})(\mathcal{T}(\mathbf{h})) \\
&= \avg \mathcal{T}(\mathbf{h})\\
&= \mathcal{T}(\mathbf{h}).
\end{aligned}    
\end{equation}
\noindent Because each $\ell_i$ is Lipschitz, there exists a constant $c \in \mathbb{R}_+$ so that
\begin{equation}
    \begin{aligned}
&\int_{[0,1]^z} \| L(\mathbf{h}) - \mathcal{T}(\mathbf{h}) \| \;\mathrm{d}\mathbf{h} \\
=&\int_{[0,1]^z}  \left \| \avg \ell(\mathbf{t}_i + \epsilon_i^\mathbf{h}) - \avg \ell(\mathbf{t}_i)   \right \| \;\mathrm{d}\mathbf{h}\\
=&\int_{[0,1]^z}  \left \| \avg \left [\ell(\mathbf{t}_i + \epsilon_i^\mathbf{h}) -  \ell(\mathbf{t}_i)   \right ] \right \| \;\mathrm{d}\mathbf{h}\\
\leq &\int_{[0,1]^z}  \avg \left \| \ell(\mathbf{t}_i + \epsilon_i^\mathbf{h}) -  \ell(\mathbf{t}_i)   \right \| \;\mathrm{d}\mathbf{h}\\
\leq &\int_{[0,1]^z} c \cdot \max_i \|\epsilon_i^\mathbf{h}\| \;\mathrm{d}\mathbf{h} \leq c \cdot \epsilon,
    \end{aligned}
\end{equation}
where $\epsilon$ is the error coming from approximating our targets, guaranteed by the Universal Approximation Theorem to be arbitrarly small as the model's size grows. At the same time $$L(\mathbf{h}^*) \approx \mathcal{T}(\mathbf{h}^*) = L^*,$$
is an (approximate) realisation of the global minimum, which  concludes the proof.
\end{proof}

From the perspective of the above proof, Theorem 1 is nothing but a special case of this construction, where we pick $\phi(\cdot) = 0$, and lose the ability to represent $\mathbf{y}_j$'s beyond mean label prediction.

As a natural consequence of this relaxation, we lose the previous guarantees of transferring to the test set, but it is still reasonable to assume this property will hold -- given that this construction relies on the separation of the loss surface section channels from the input processing ones.
Also, this construction still does benefit from the learnability property, since it again relies solely on recasting the problem as regular supervised learning.

\section{Conclusions}

In this paper, we provided a proof of a somewhat surprising property first empirically observed by \citet{skorokhodov2019loss}: The loss surfaces of deep neural networks contain every low-dimensional pattern and
\begin{itemize}
    \item this property holds for any dataset,
    \item the pattern locations transfer from train to test set as well as to other datasets with the same loss and $P(\mathbf{y})$,
    \item finding such patterns is not harder than regular supervised learning,
    \item the patterns can be guaranteed to be axis aligned,
    \item the patterns can be modified to have loss value epsilon away from global minima of the original problem.
\end{itemize}

Whilst to the best of the authors' knowledge the results presented here do not directly lead to practical application, they do add to our overall understanding of neural network loss surfaces. 
In particular, our results can be seen as a cautionary note if considering adding regularisation in the form of local loss surface geometry preferences on low-dimensional sections, since networks might be able to \emph{cheat} and satisfy these constraints independently from their operation over input space.

\section*{Acknowledgements}
We would like to thank Balaji Lakshminarayanan and Tom Erez for helpful discussion and comments.

%% file: main.bbl
\begin{thebibliography}{5}
\providecommand{\natexlab}[1]{#1}
\providecommand{\url}[1]{\texttt{#1}}
\expandafter\ifx\csname urlstyle\endcsname\relax
  \providecommand{\doi}[1]{doi: #1}\else
  \providecommand{\doi}{doi: \begingroup \urlstyle{rm}\Url}\fi

\bibitem[Choromanska et~al.(2015)Choromanska, LeCun, and
  Arous]{choromanska2015open}
A.~Choromanska, Y.~LeCun, and G.~B. Arous.
\newblock Open problem: The landscape of the loss surfaces of multilayer
  networks.
\newblock In \emph{Conference on Learning Theory}, pages 1756--1760, 2015.

\bibitem[Fort and Jastrzebski(2019)]{fort2019large}
S.~Fort and S.~Jastrzebski.
\newblock Large scale structure of neural network loss landscapes.
\newblock \emph{Advances in Neural Information Processing Systems}, 2019.

\bibitem[Hornik(1991)]{hornik1991approximation}
K.~Hornik.
\newblock Approximation capabilities of multilayer feedforward networks.
\newblock \emph{Neural networks}, 4\penalty0 (2):\penalty0 251--257, 1991.

\bibitem[Li et~al.(2018)Li, Xu, Taylor, Studer, and
  Goldstein]{li2018visualizing}
H.~Li, Z.~Xu, G.~Taylor, C.~Studer, and T.~Goldstein.
\newblock Visualizing the loss landscape of neural nets.
\newblock In \emph{Advances in Neural Information Processing Systems}, pages
  6389--6399, 2018.

\bibitem[Skorokhodov and Burtsev(2019)]{skorokhodov2019loss}
I.~Skorokhodov and M.~Burtsev.
\newblock Loss surface sightseeing by multi-point optimization.
\newblock \emph{Advances in Neural Information Processing Systems Workshop
  "Beyond First Order Methods in ML"}, 2019.

\end{thebibliography}
